\documentclass[11pt, onecolumn]{article}
\usepackage{hyperref}
\usepackage{url}

\usepackage{float}
\usepackage{stmaryrd} \usepackage{url} \usepackage[latin1]{inputenc}
\usepackage{graphicx}
\usepackage{amssymb} \usepackage{subfigure} \usepackage{amsmath}
\usepackage{amsthm}
\usepackage{dcolumn}
\usepackage{bm}
\usepackage{color}
\usepackage{mathtools}

\usepackage[noline,boxed]{algorithm2e}

\newtheorem{theorem}{Theorem}[section]
\newtheorem{statement}[theorem]{Statement}

\newtheorem{lemma}[theorem]{Lemma}

\newtheorem{remark}[theorem]{Remark}

\newtheorem{claim}[theorem]{Claim}

\newtheorem*{subproblem*}{Subproblem}
\newtheorem{definition}[theorem]{Definition}

\usepackage{hyperref}

\newcommand{\mc}{\mathcal}
 \newcommand{\mbb}{\mathbb} \newcommand{\wt}{\widetilde}
\newcommand{\fa}{\forall}  \newcommand{\tb}{\textbf}
\newcommand{\tx}{\text}

\newcommand{\dr}{\text{d}}

\newcommand{\ot}{\otimes}

\newcommand{\R}{\mbb R}
\newcommand{\C}{\mbb C}
\newcommand{\Z}{\mbb Z}

\newcommand{\wh}{\widehat}

\newcommand{\diag}{\tx{diag}}
\newcommand{\poly}{\tx{poly}}
\newcommand{\cond}{\tx{cond}}

\newcommand{\lt}{\left}
\newcommand{\rt}{\right}

\usepackage{enumitem}
\usepackage[margin=1in]{geometry}

\newcommand{\sk}[1]{}
\newcommand{\qt}[1]{}
\newcommand{\qq}[1]{{\color{black}{#1}}}

\author{Qingqing Huang
  \\  MIT,
  \\  EECS,
  \\  LIDS,
  \\  \texttt{qqh@mit.edu}
  \and  Sham M. Kakade
  \\ University of Washington,
  \\ Department of Statistics,
  \\Computer Science \& Engineering,
  \\  \texttt{sham@cs.washington.edu}
}

\date{}

\begin{document}
\title{Super-Resolution Off the Grid }

\maketitle
\qt{change to complex amplitude, done}

\begin{abstract}
  Super-resolution is the problem of recovering a superposition of point sources using bandlimited
  measurements, which may be corrupted with noise. This signal processing problem arises in numerous
  imaging problems, ranging from astronomy to biology to spectroscopy, where it is common to take
  (coarse) Fourier measurements of an object.
  Of particular interest is in obtaining estimation procedures which are robust to noise, with the
  following desirable statistical and computational properties: we seek to use coarse Fourier
  measurements (bounded by some \emph{cutoff frequency}); we hope to take a (quantifiably) small
  number of measurements; we desire our algorithm to run quickly.

  Suppose we have $k$ point sources in $d$ dimensions, where the points are separated by at least
  $\Delta$ from each other (in Euclidean distance). This work provides an algorithm with the
  following favorable guarantees:
\begin{itemize}
\item The algorithm uses Fourier measurements, whose frequencies are bounded by $O(1/\Delta)$ (up to
  log factors). Previous algorithms require a \emph{cutoff frequency} which may be as large as
  $\Omega(\sqrt{d}/\Delta)$.
\item The number of measurements taken by and the computational complexity of our algorithm are
  bounded by a polynomial in both the number of points $k$ and the dimension $d$, with \emph{no}
  dependence on the separation $\Delta$. In contrast, previous algorithms depended inverse
  polynomially on the minimal separation and exponentially on the dimension for both of these
  quantities.
\end{itemize}
Our estimation procedure itself is simple: we take random bandlimited measurements (as opposed to
taking an exponential number of measurements on the hyper-grid). Furthermore, our analysis and
algorithm are elementary (based on concentration bounds for sampling
and the singular value
decomposition).
\end{abstract}

\setcounter{page}{1}

\newpage \section{Introduction}

We follow the standard mathematical abstraction of this problem (Candes \& Fernandez-Granda
\cite{candes2014towards,candes2013super}): consider a $d$-dimensional signal $x(t)$ modeled as a
weighted sum of $k$ Dirac measures in $\R^d$:
\begin{align}
  \label{eq:sig-x}
  x(t) = \sum_{j=1}^{k} w_j \delta_{\mu^{(j)}},
\end{align}
where the point sources, the $\mu^{(j)}$'s, are in $\R^d$.
%
Assume that the \qq{weights $w_j$ are complex valued, whose absolute values} are lower and upper bounded
by some positive constant.
Assume that we are given $k$, the number of point sources\footnote{An upper bound of the number of
  point sources suffices.}.

Define the measurement function $f(s):\R^d\to \C$ to be the convolution of the point source $x(t)$
with a low-pass point spread function $e^{i\pi <s, t> }$ as below:
\begin{align}
  \label{eq:measure-f}
  f(s) = \int_{t\in \R^d} e^{i\pi < t,s>} x(\dr t) = \sum_{j=1}^{k} w_j e^{i\pi<\mu^{(j)},s>}.
\end{align}
In the noisy setting, the measurements are corrupted by uniformly bounded perturbation $z$:
\begin{align}
  \label{eq:measure-noise}
  \wt f(s) = f(s) + z(s), \quad |z(s)|\le \epsilon_z, \fa s.
\end{align}

Suppose that we are only allowed to measure the signal $x(t)$ by evaluating the measurement function
$\wt f(s)$ at any $s\in\R^d$, and we want to recover the parameters of the point source signal,
i.e., $\{w_j, \mu^{(j)}: j\in[k]\}$.
We follow the standard normalization to assume that:
\[
\mu^{(j)}\in  [-1,+1]^d,\quad \qq{ |w_j|} \in [0,1] \quad \fa j\in[k].
\]
Let $w_{min}=\min_{j}|w_j|$ denote the minimal weight, and let $\Delta$ be the minimal separation of
the point sources defined as follows:
\begin{align}
  \label{eq:def-delta}
  \Delta = \min_{j\neq j'}\|\mu^{(j)}-\mu^{(j')}\|_2,
\end{align}
where we use the Euclidean distance between the point sources for ease of exposition\footnote{Our
  claims hold withut using the ``wrap around metric'', as
  in~\cite{candes2014towards,candes2013super}, due to our random sampling.  Also, it is possible to
  extend these results for the $\ell_p$-norm case.}.  These quantities are key parameters in our
algorithm and analysis.  Intuitively, the recovery problem is harder if the minimal separation
$\Delta$ is small and the minimal weight $w_{min}$ is small.

The first question is that, given exact measurements, namely $\epsilon_z=0$, where and how many
measurements should we take so that the original signal $x(t)$ can be exactly recovered.
\begin{definition}[Exact recovery]
  In the exact case, i.e. $\epsilon_z=0$, we say that an algorithm achieves exact recovery with $m$
  measurements of the signal $x(t)$ if, upon input of these $m$ measurements, the algorithm returns
  the exact set of parameters $\{w_j, \mu^{(j)}: j\in[k]\}$.
\end{definition}

Moreover, we want the algorithm to be measurement noise tolerant, in the sense that in the presence
of measurement noise we can still recover good estimates of the point sources.
\begin{definition}[Stable recovery]
  In the noisy case, i.e., $\epsilon_z \ge 0$, we say that an algorithm achieves stable recovery
  with $m$ measurements of the signal $x(t)$ if, upon input of these $m$ measurements, the algorithm
  returns estimates $\{\wh w_j, \wh \mu^{(j)}: j\in[k]\}$ such that
  \begin{align*}
    \min_{\pi} \max\lt\{ \|\wh \mu^{(j)} - \mu^{(\pi(j))} \|_2: j\in[k]\rt\} \le \poly(d,k)
    \epsilon_z,
  \end{align*}
  where the $\min$ is over permutations $\pi$ on $[k]$ and \poly(d,k) is a polynomial function in
  $d$ and $k$.
\end{definition}
By definition, if an algorithm achieves stable recovery with $m$ measurements, it also
achieves exact recovery with these $m$ measurements.
\sk{this isn't quite true due to the
  quantification of $m$, as it may be a function of
  $\epsilon_z$. this might have been my fault in definition changes.}\qt{done}

The terminology of ``super-resolution'' is appropriate due to the following remarkable result (in
the noiseless case) of Donoho \cite{donoho1992superresolution}: suppose we want to accurately
recover the point sources to an error of $\gamma$, where $\gamma \ll \Delta$. Naively, we may expect
to require measurements whose frequency depends inversely on the desired the accuracy $\gamma$.
Donoho \cite{donoho1992superresolution} showed that it suffices to obtain a finite number of
measurements, whose frequencies are bounded by $O(1/\Delta)$, in order to achieve \emph{exact}
recovery; thus resolving the point sources far more accurately than that which is naively implied by
using frequencies of $O(1/\Delta)$. Furthermore, the work of Candes \& Fernandez-Granda
\cite{candes2014towards,candes2013super} showed that stable recovery, in the univariate case
($d=1$), is achievable with a cutoff frequency of $O(1/\Delta)$ using a convex program and a number
of measurements whose size is polynomial in the relevant quantities.

\subsection{This work}

We are interested in stable recovery procedures with the following desirable statistical and
computational properties: we seek to use coarse (low frequency) measurements; we hope to take a
(quantifiably) small number of measurements; we desire our algorithm run quickly. Informally, our
main result is as follows:

\begin{theorem}[Informal statement of Theorem~\ref{thm:main-thm}]
  For a fixed probability of error, the proposed algorithm achieves stable recovery with a number of
  measurements and with computational runtime that are both on the order of \qq{$O(
  (k\log(k)+d)^2)$}. Furthermore, the algorithm makes measurements which are bounded in frequency by
  $O(1/\Delta)$ (ignoring log factors).
\end{theorem}

Notably, our algorithm and analysis directly deal with the multivariate case, with the univariate
case as a special case. Importantly, the number of measurements and the computational runtime do
\emph{not} depend on the minimal separation of the point sources. This may be important even in
certain low dimensional imaging applications where taking physical measurements are costly (indeed,
super-resolution is important in settings where $\Delta$ is small).  Furthermore, our technical
contribution of how to decompose a certain tensor constructed with Fourier measurements may be of
broader interest to related questions in statistics, signal processing, and machine learning.

\begin{table}
  \renewcommand{\arraystretch}{1.7}
  \hspace*{-0.25in}
  \begin{tabular}[]{|c|c|c|c|c|c|c|}
    \hline
    &   \multicolumn{3}{c|}{ $d=1$}  & \multicolumn{3}{c|}{$d\ge 1$}
    \\
    \hline
    &  cutoff freq & measurements & runtime  &  cutoff freq & measurements & runtime
    \\
    \hline
    SDP & $1\over \Delta$ & $k\log(k)\log({1\over \Delta})$ & $poly({1\over
      \Delta},k)$
    & ${C_d\over \Delta_{\infty}}$ & $({1\over \Delta_{\infty}})^d$ &  $poly(({1\over \Delta_{\infty}})^d,k)$
    \\
    \hline
    MP & $1\over \Delta$ &  $1\over \Delta$ &  $({1\over \Delta})^3$
    & - & - & -
    \\
    \hline
    \tb{Ours} & $1\over \Delta$ & \qq{$(k\log(k) )^2$} & \qq{$(k\log(k) )^2$}
    & ${\log(kd)}\over \Delta$ & \qq{$(k\log(k)+d)^2$} & \qq{$(k\log(k)+d)^2$}
    \\
    \hline
  \end{tabular}
  \caption{
    \sk{fix the table spacing as $\Delta$ is cutoff}\qt{done?}
    See Section~\ref{sec:comp-relat-work} for description.
    See Lemma~\ref{rem:cutoff} for details about the cutoff frequency. Here, we are implicitly using $O(\cdot)$ notation.
  }
  \label{tab:a}
\end{table}

\subsection{Comparison to related work}
\label{sec:comp-relat-work}
Table~\ref{tab:a} summarizes the comparisons between our algorithm and the existing results. The
multi-dimensional cutoff frequency we refer to in the table is the maximal coordinate-wise entry of
any measurement frequency $s$ (i.e. $\| s\|_\infty$). ``SDP'' refers to the semidefinite
programming (SDP) based algorithms of Candes \& Fernandez-Granda
\cite{candes2013super,candes2014towards}; in the univariate case, the number of measurements can be
reduced by the method in Tang et. al.  \cite{tang2013compressed} (this is reflected in the table).
``MP'' refers to the matrix pencil type of methods, studied in \cite{liao2014music} and
\cite{moitra2014threshold} for the univariate case.  Here, we are defining the infinity norm
separation as $\Delta_{\infty} = \min_{j\neq j'}\|\mu^{(j)}-\mu^{(j')}\|_\infty$, which is
understood as the wrap around distance on the unit circle.  $C_d\geq 1$ is a problem dependent
constant (discussed below).

Observe the following differences between our algorithm and prior work:
\begin{enumerate}
  [labelindent=*,leftmargin=*,label= \arabic*)]
\item Our minimal separation is measured under the $\ell_2$-norm instead of the infinity norm, as in
  the SDP based algorithm. Note that $\Delta_{\infty}$ depends on the coordinate system; in the
  worst case, it can underestimate the separation by a $1/ \sqrt{d}$ factor, namely
  $\Delta_{\infty}\sim \Delta/\sqrt{d}$.

\item The computation complexity and number of measurements are polynomial in dimension $d$ and the
  number of point sources $k$, and surprisingly do not depend on the minimal separation of the point
  sources!  Intuitively, when the minimal separation between the point sources is small, the problem
  should be harder, this is only reflected in the sampling range and the cutoff frequency of the
  measurements in our algorithm.

\item Furthermore, one could project the multivariate signal to the coordinates and solve multiple
  univariate problems (such as in \cite{potts2013parameter, nandi2013noise}, which provided only
  exact recovery results). Naive random projections would lead to a cutoff frequency of
  $O(\sqrt{d}/\Delta)$.
\end{enumerate}

{\bf SDP approaches:} The work in \cite{candes2013super, candes2014towards, fernandez2014convex}
formulates the recovery problem as a total-variation minimization problem; they then show the dual
problem can be formulated as an SDP. They focused on the analysis of $d=1$ and only explicitly
extend the proofs for $d=2$. For $d\geq 1$, Ingham-type theorems (see
\cite{russell1978controllability, komornik2005fourier}) suggest that $C_d = O(\sqrt{d})$.

The number of measurements can be reduced by the method in \cite{tang2013compressed} for the $d=1$
case, which is noted in the table. Their method uses sampling ``off the grid''; technically, their
sampling scheme is actually sampling random points from the grid, though with far fewer
measurements.

{\bf Matrix pencil approaches:} The matrix pencil method, MUSIC and Prony's method are essentially
the same underlying idea, executed in different ways.  The original Prony's method directly attempts
to find roots of a high degree polynomial, where the root stability has few guarantees. Other
methods aim to robustify the algorithm.

\sk{We should check that the MP algorithm actually gets super
  resolution/exact recovery. The
  result of Moitra, due to the dependence on the noise on the, means
  no exact recovery. This is bad.}\qt{todo}

Recently, for the univariate matrix pencil method, Liao \& Fannjiang \cite{liao2014music} and Moitra
\cite{moitra2014threshold} provide a stability analysis of the MUSIC algorithm. Moitra
\cite{moitra2014threshold} studied the optimal relationship between the cutoff frequency and
$\Delta$, showing that if the cutoff frequency is less than ${1/ \Delta}$, then stable recovery is
not possible with matrix pencil method (with high probability).

\subsection{Notation}

Let $\R$, $\C$, and $\Z$ to denote real, complex, and natural numbers.
For $d\in\Z$, $[d]$ denotes the set $[d] =\{1,\dots, d\}$.
For a set $\mc S$, $|\mc S|$  denotes its cardinality.  We use $\oplus$ to denote the direct
sum of sets, namely $\mc S_1\oplus \mc S_2 = \{(a+b): a\in\mc S_1, b\in\mc S_2\}$.

%
Let $e_n$ to denote the $n$-th standard basis vector in $\R^d$, for $n\in[d]$.
Let $\mc P_{R,2}^d = \{x\in\R^{d}: \|x\|_2= 1\}$ to denote the $d$-sphere of radius $R$ in the
$d$-dimensional standard Euclidean space.
%

Denote the condition number of a matrix $X\in\R^{m\times n}$ as
$\cond_2(X) = \sigma_{max}(X) /  \sigma_{min}(X)$, where $\sigma_{max}(X)$ and $\sigma_{min}(X)$
are the maximal and minimal singular values of $X$.

\sk{should define the projection notation a little more carefully.}\qt{done}


We use $\ot$ to denote tensor product. Given matrices $A, B,C\in\C^{m\times k}$, the tensor product
$V= A\ot B\ot C\in\C^{m\times m\times m}$ is equivalent to $V_{i_1,i_2,i_3} = \sum_{n=1}^{k}A_{i_1,n}B_{i_2,n}
C_{i_3,n}$.
Another view of tensor is that it defines a multi-linear mapping. For given dimension $m_A, m_B,m_C$
the mapping $V(\cdot,\cdot,\cdot):\C^{m\times m_A} \times \C^{m\times m_B} \times \C^{m\times m_C}
\to \C^{m_A\times m_B\times m_C}$ is defined as:
$$[V(X_A, X_B, X_c)]_{i_1,i_2,i_3} = \sum_{j_1,j_2,j_3\in[m]}V_{j_1,j_2,j_3} [X_{A}]_{j_1, i_1} [X_{B}]_{j_2, i_2} [X_{C}]_{j_3, i_3}.$$
In particular, for $a\in\C^{m}$, we use $V(I,I,a)$ to denote the projection of tensor $V$ along the
3rd dimension.
Note that if the tensor admits a decomposition $V = A\ot B\ot C$, it is straightforward to verify
that $$V(I,I,a) = A Diag(C^\top a) B^\top.$$
It is well-known that if the factors $A, B, C$ have full column rank then the rank $k$ decomposition
is unique up to re-scaling and common column permutation. Moreover, if the condition number of the
factors are upper bounded by a positive constant, then one can compute the unique tensor
decomposition $V$ with stability guarantees (See \cite{anandkumar2014tensor} for a
review. Lemma~\ref{lem:stable-jennrich} herein provides an explicit statement.).

\section{Warm-up}

\subsection{1-D case: revisiting the matrix pencil method}

Let us first review the matrix pencil method for the univariate case, which stability was recently
rigorously analyzed in Liao \& Fannjiang \cite{liao2014music} and Moitra \cite{moitra2014threshold}.

\sk{should say that matrices with this structure and Hankel and Vandermonde}\qt{done}

A square matrix $H$ is called a \emph{Hankel} matrix if its skew-diagonals are constants, namely
$H_{i,j} = H_{i-1,j+1}$.
For some positive constants $m\in\Z$, sample to get the measurements $f(s)$ evaluated at the
sampling set $\mc S_3 = \{0, 1, \dots, 2m\}$, and construct two {Hankel matrices} $H_0,
H_1\in\C^{m\times m}$:
  \begin{align}
    \label{eq:H0H1-def}
H_0 = \left[
  \begin{array}[c]{cccc}
    f(0) & f(1) & \dots & f(m-1)
    \\
    f(1) & f(2) &\dots & f(m)
    \\
    \vdots & & & \vdots
    \\
    f(m-1) &f(m) &\dots & f(2m-1)
  \end{array}
\right],
\quad
H_1 = \left[
  \begin{array}[c]{cccc}
    f(1) & f(2) & \dots & f(m)
    \\
    f(2) & f(3) &\dots & f(m+1)
    \\
    \vdots & & & \vdots
    \\
    f(m) &f(m+1) &\dots & f(2m)
  \end{array}
\right].
\end{align}

Define \qq{$D_w\in\C_{diag}^{k\times k}$} to be the diagonal matrix with the weights on the main
diagonal: $[D_w]_{j,j} = w_j$. Define $D_\mu\in\C_{diag}^{k\times k}$ to be $[D_\mu]_{j,j} =
e^{i\pi\mu^{(j)}}$.

A matrix $V$ is called a \emph{Vandermonde matrix} if each column is a geometric progression.
defined the Vandermonde matrix $V_m \in\C^{m\times k}$ as below:
\begin{align}
  \label{eq:def-vand}
  V_m = \left[
    \begin{array}[c]{ccc}
      1 & \dots & 1
      \\
      (e^{i\pi\mu^{(1)}})^1 &  \dots &  (e^{i\pi\mu^{(k)}})^1
      \\
      \vdots & & \vdots
      \\
      (e^{i\pi\mu^{(1)}})^{m-1} &  \dots &  (e^{i\pi\mu^{(k)}})^{m-1}
  \end{array}
\right].
\end{align}

The two Hankel matrices $H_0$ and $H_1$ admit the following simultaneous diagonalization:
\begin{align}
  \label{eq:H01-decomp}
  H_0 = V_m D_w V_m^\top, \quad H_1 = V_m D_w D_\mu V_m^\top.
\end{align}
As long as $V_m$ is of full rank, this simultaneous diagonalization can be computed by solving the
generalized eigenvalue problem, and the parameters of the point source can thus be obtained from the
factor $V_m$ and $D_w$.

The univariate matrix pencil method only needs $m\ge k$ to achieve exact recovery.  In the noisy case,
the stability of generalized eigenvalue problem depends on the condition number of the Vandermonde
matrix $V_m$ and the minimal weight $w_{min}$.

Since all the nodes ($e^{i\pi\mu^{(j)}}$'s) of this Vandermonde matrix lie on the unit circle in the
complex plane, it is straightforward to see that asymptotically $\lim_{m\to\infty}\cond_2(V_m) =
1$. Furthermore, for $m > {1/ \Delta}$, \cite{liao2014music,moitra2014threshold} showed that
$\cond_2(V_m)$ is upper bounded by a constant that does not depend on $k$ and $m$.  This bound on
condition number is also implicitly discussed in \cite{potts2013parameter}.

Another way to view the matrix pencil method is that it corresponds to
the low rank 3rd order tensor decomposition (see for example
\cite{anandkumar2014tensor}). This view will help us generalize matrix
pencil method to higher dimension $d$ in a direct way, without
projecting the signal on each coordinate and apply the univariate
algorithm multiple times.
For $m\ge k$, construct a 3rd order tensor $F \in\C^{m\times m\times 2}$ with elements of $H_0$ and
$H_1$ defined in \eqref{eq:H0H1-def} as:
\begin{align*}
  F_{i,i',j} = [H_{j-1}]_{i,i'}, \quad \fa j\in [2], i,i'\in[m].
\end{align*}
Note that the two slices along the 3rd dimension of $F$ are $H_0$ and $H_1$. Namely $F(I,I,e_1) =
H_0$, and $F(I,I,e_2) = H_1$.
Recall the matrix decomposition of $H_0$ and $H_1$ in \eqref{eq:H01-decomp}.  Since $m\ge k$ and the
$\mu^{(j)}$'s are distinct, we know that $F$ has the \emph{unique} rank $k$ tensor decomposition:
  $$F = V_m\ot V_m\ot (V_2 D_w).$$

  Given the tensor $F$, the basic idea of the well-known Jennrich's algorithm
  (\cite{harshman1970foundations,leurgans1993decomposition}) for finding the unique low rank tensor
  decomposition is to consider two random projections $v_1, v_2\in\R^{m}$, and then with high
  probability the two matrices $F(I,I,v_1)$ and $F(I,I,v_2)$ admit simultaneous diagonalization.
  Therefore, the matrix pencil method is indeed a special case of Jennrich's algorithm by setting
  $v_1 = e_1$ and $v_2 = e_2$

\sk{maybe make the above a little more clear. some people may not be all that
  familiar with Tensor notation.}\qt{done}

\subsection{The multivariate case: a toy example}

One could naively extend the matrix pencil method to higher dimensions by using taking measurements
from a hyper-grid, which is of size exponential in the dimension $d$. We now examine a toy problem
which suggests that the high dimensional case may not be inherently more difficult than the
univariate case.

The key ideas is that an appropriately sampled set can significantly reduce the number of
measurements (as compared to using all the grid points).  Tang et al \cite{tang2013compressed} made
a similar observation for the univariate case. They used a small random subset of measurements
(actually still from the grid points) and showed that this contains enough information to recover
all the measurement on the grid; the full measurements were then used for stably recovering the
point sources.

Consider the case where the dimension $d\ge k$.  \qq{Assume that $w_j$'s are real valued,} and for all
$j\in[k]$ and $n\in[d]$, the parameters $\mu^{(j)}_n$ are i.i.d. and uniformly distributed over
$[-1,+1]$.
This essentially corresponds to the standard ($L_2$) incoherence conditions (for the
$\mu^{(j)}$'s).~\footnote{ This setting is different from the 2-norm separation condition. To see
  the difference, note that the toy algorithm does not work for constant shift $\mu^{(1)} =
  \mu^{(2)}+ \Delta$. This issue is resolved in the general algorithm, when the condition is stated
  in terms of 2-norm separation.}  The following simple algorithm achieves stability with polynomial
complexity.

First, take $d^3$ number of measurements by evaluating $f(s)$ in the set $\mc S_3 = \{ s = e_{n_1} +
e_{n_2} + e_{n_3}: [n_1,n_2,n_3]\in[d]\times [d]\times [d] \}$, noting that $\mc S_3$ contains only
a subset of $d^3$ points from the grid of $[3]^d$. Then, construct a 3rd order tensor
$F\in\C^{d\times d\times d}$ with the measurements in the following way:
\begin{align*}
  F_{n_1,n_2,n_3} = f(s)\big |_{s = e_{n_1} + e_{n_2} + e_{n_3}}, \quad \fa n_1,n_2,n_3\in[d].
\end{align*}
Note that the measurement $f(e_{1} + e_{2} + e_{3}) = \sum_{j=1}^{k}w_j e^{i\pi (\mu^{(j)}_1
  +\mu^{(j)}_2 +\mu^{(j)}_3 ) } = \sum_{j=1}^{k}w_j e^{i\pi \mu^{(j)}_1} e^{i\pi\mu^{(j)}_2}
e^{i\pi\mu^{(j)}_3 }.$ It is straightforward to verify that $F$ has a rank-$k$ tensor factorization
$ F = V_d \ot V_d \ot( V_d D_w)$, where the factor $V_d\in\R^{d\times k}$ is given by:
\begin{align}
  \label{eq:Vd}
  V_d = \left[
    \begin{array}[c]{ccc}
        e^{i\pi \mu^{(1)}_1} & \dots & e^{i\pi \mu^{(k)}_1}
       \\ e^{i\pi \mu^{(1)}_2} & \dots & e^{i\pi \mu^{(k)}_2}
       \\ \vdots & \dots & \vdots
       \\ e^{i\pi \mu^{(1)}_d} & \dots & e^{i\pi \mu^{(k)}_d}
    \end{array}
  \right].
\end{align}
Under the distribution assumption of the point sources, the entries $e^{i\pi \mu^{(j)}_n}$ are
i.i.d. and uniformly distributed over the unit circle on the complex plane.  Therefore almost surely
the factor $V_d$ has full column rank, and thus the tensor decomposition is unique. Moreover here
$w_j$'s are real and each element of $V_S$ has unit norm, we have a rescaling constraint with the
tensor decomposition, with which we can uniquely obtain the factor $V_S$ and the weights in $D_w$.
By taking element-wise log of $V_S$ we can read off the parameters of the point sources from $V_S$
directly.
Moreover, with high probability, we have  that  $\cond_2(V_d)$ concentrates around 1,  thus
the simple algorithm achieves stable recovery.

\section{Main Results}

\subsection{The algorithm}

\begin{algorithm}[t!]
  \caption{General algorithm}
  \label{alg:gen}
  \DontPrintSemicolon

  \tb{Input:}  $R$, $m$, noisy measurement function $\wt f(\cdot)$.
  \BlankLine
  \tb{Output:} Estimates $\{\wh w_j, \wh \mu^{(j)}:j\in[k]\} $.
  \BlankLine
  \begin{enumerate}
        [labelindent=*,leftmargin=*,rightmargin=\dimexpr\linewidth-16cm]
      \item {\bf Take measurements:}

        Let $\mc S =\{s^{(1)},\dots, s^{(m)}\}$ be $m$ i.i.d. samples from
    the Gaussian distribution $\mc N(0, R^2 I_{d\times d})$.  Set $s^{(m+n)}=e_n$ for all $n\in[d]$
    and $s^{(m+n+1)} =0$. Denote $m'=m+d+1$.

    \qq{Take another random samples $v$ from the unit sphere, and set $v^{(1)}=v$ and $v^{(2)}=2v$.
      Construct a tensor $\wt F\in\C^{m'\times m'\times 3}$: $\wt F_{n_1,n_2,n_3} = \wt f(s)\big |_{s =
        s^{(n_1)} + s^{(n_2)} + v^{(n_3)}}$.}
  \item {\bf Tensor Decomposition:} Set $(\wh V_{S'}, \wh D_{w}) =$ TensorDecomp($\wt F$).

    For $j=1,\dots, k$, set $[\wh V_{S'}]_{j} =[ \wh V_{S'}]_j/[\wh V_{S'}]_{m',j}$
  \item {\bf Read of estimates:} For $j=1,\dots, k$, set $\wh \mu^{(j)} = Real(\log([\wh V_S]_{[m+1:m+d,j]})/( i
    \pi)).$


  \item \qq{Set  $\wh W = \arg\min_{ W\in \C^{k}} \|\wh F- \wh V_{S'}\ot \wh V_{S'}\ot \wh V_{d}D_w\|_{F}$}.


  \end{enumerate}
  \BlankLine

\end{algorithm}

We briefly describe the steps of Algorithm~\ref{alg:gen} below:
\begin{enumerate}
    [labelindent=*,leftmargin=*,label=]
  \item \tb{(Take measurements)} Given positive numbers $m$ and $R$, randomly draw a sampling set
    $\mc S=\lt\{s^{(1)}, \dots s^{(m)}\rt\}$ of $m$ i.i.d. samples of the Gaussian distribution $\mc
    N(0, R^2 I_{d\times d})$.
    Form the set \qq{$\mc S'=\mc S\cup\{ s^{(m+1)} =e_1,\dots, s^{(m+d)}=e_d, s^{(m+d+1)}=0\}\subset \R^d$.}  Denote
    $m'=m+d+1$.
    \qq{Take another independent random sample $v$ from the unit sphere, and define $v^{(1)}=v$,
      $v^{(2)}=2v$.}
    Construct the 3rd order tensor \qq{ $\wt F\in\C^{m'\times m'\times 3}$} with noise corrupted
    measurements $\wt f(s)$ evaluated at the points in \qq{$ \mc S' \oplus \mc S' \oplus
      \{v^{(1)},v^{(2)}\}$}, arranged in the following way:
    \begin{align}
      \label{eq:tensor-F}
      \wt F_{n_1,n_2,n_3} = \wt f(s)\big |_{s = s^{(n_1)} + s^{(n_2)} + v^{(n_3)}}, \fa
      n_1,n_2\in[m'], n_3\in[2].
    \end{align}

  \item \tb{(Tensor decomposition)} Define the \emph{characteristic matrix} $V_S$ to be:
    \begin{align}
      \label{eq:Vs-general}
      V_S = \left[
        \begin{array}[c]{ccc}
          e^{i\pi<\mu^{(1)}, s^{(1)}>} & \dots & e^{i\pi <\mu^{(k)}, s^{(1)}>}
          \\ e^{i\pi <\mu^{(1)}, s^{(2)}>} & \dots & e^{i\pi <\mu^{(k)}, s^{(2)}>}
          \\ \vdots & \dots & \vdots
          \\ e^{i\pi <\mu^{(1)}, s^{(m)}>} & \dots & e^{i\pi <\mu^{(k)},s^{(m)}>}
        \end{array}
      \right].
    \end{align}
    and define matrix $V'\in\C^{m'\times k}$ to be
    \begin{align}
      \label{eq:Vs-p}
      V_{S'} = \lt[
      \begin{array}[c]{c}
        V_S\\ V_d \\ 1,\dots,1
      \end{array}
      \rt],
    \end{align}
    where $V_d\in\C^{d\times k}$ is defined in \eqref{eq:Vd}.
    Define
    \begin{align*}
      V_2 = \left[
        \begin{array}[c]{ccc}
          e^{i\pi<\mu^{(1)}, v^{(1)}>} & \dots & e^{i\pi <\mu^{(k)}, v^{(1)}>}
          \\ e^{i\pi <\mu^{(1)}, v^{(2)}>} & \dots & e^{i\pi <\mu^{(k)}, v^{(2)}>}
          \\
          1 & \dots & 1
        \end{array}
      \right].
    \end{align*}

    Note that in the exact case ($\epsilon_z=0$) the tensor $F$ constructed in \eqref{eq:tensor-F}
    admits a rank-$k$ decomposition:
    \begin{align}
      \label{eq:F-decomp}
      F= V_{S'} \ot V_{S'}\ot (V_{2} D_w),
    \end{align}
    Assume that $V_{S'}$ has full column rank, then this tensor decomposition is unique up to column
    permutation and rescaling \qq{with very high probability over the randomness of the random unit
      vector $v$.  Since each element of $V_{S'}$ has unit norm, and we know that the last row of
      $V_{S'}$ and the last row of $V_2$ are all ones,} there exists a proper scaling so that we can
    uniquely recover $w_j$'s and columns of $V_{S'}$ up to common permutation.

    In this paper, we adopt Jennrich's algorithm (see Algorithm~\ref{alg:tensordecomp}) for tensor
    decomposition. Other algorithms, for example tensor power method (\cite{anandkumar2014tensor})
    and recursive projection (\cite{vempala2014max}), which are possibly more stable than Jennrich's
    algorithm, can also be applied here.

  \item \tb{(Read off estimates)} Let $\log(V_{d})$ denote the element-wise logarithm of $V_{d}$.
    The estimates of the  point sources are given by:
    \begin{align*}
      \lt[\mu^{(1)}, \mu^{(2)}, \dots, \mu^{(k)}\rt] = { \log(V_d)\over i \pi}.
    \end{align*}

\end{enumerate}

\begin{remark}
  In the toy example, the simple algorithm corresponds to using the sampling set $\mc S' = \{e_1,\dots,
  e_d\}$.  The conventional univariate matrix pencil method corresponds to using the sampling set
  $\mc S' = \{0, 1, \dots, m\}$ and the set of measurements $\mc S'\oplus\mc S'\oplus \mc S'$
  corresponds to the grid $[m]^3$.
\end{remark}

\begin{algorithm}[t!]
  \caption{TensorDecomp}
  \label{alg:tensordecomp}
  \DontPrintSemicolon

  \tb{Input:}  Tensor \qq{$\wt F\in\C^{m\times m\times 3}$}, rank $k$.
  \BlankLine
  \tb{output:}  Factor $\wh V\in\C^{m\times k}$.
  \BlankLine
  \begin{enumerate}
    [labelindent=*,leftmargin=*,rightmargin=\dimexpr\linewidth-16cm]
  \item Compute the truncated SVD of $\qq{\wt F(I,I,e_1)} = \wh P \wh \Lambda \wh P^\top$ with the $k$
    leading singular values.
  \item Set $\wh E =\wt F(\wh P,\wh P,I)$. Set $\wh E_1 = \wh E(I,I,e_1)$ and $\wh E_2 = \wh E(I,I,e_2)$.
  \item Let the columns of $\wh U$ be the eigenvectors of $\wh E_1\wh E_2^{-1}$
    corresponding to the $k$ eigenvalues with the largest absolute value.
  \item Set $\wh V = \sqrt{m}\wh P\wh U$.
  \end{enumerate}
  \BlankLine

\end{algorithm}

\subsection{Guarantees}

In this section, we discuss how to pick the two parameters $m$ and $R$ and prove that the proposed
algorithm indeed achieves stable recovery in the presence of measurement noise.
\begin{theorem}[Stable recovery]
  \label{thm:main-thm}
  There exists a universal constant $C$ such that the following holds.

  Fix $\epsilon_x, \delta_s, \delta_{v} \in(0,{1\over 2})$;

  pick $m$ such that $ m\ge \max\lt\{ {k\over \epsilon_x}\sqrt{8\log{k\over \delta_s}},\ \ d\rt\}$;

  for $d=1$, pick $R\ge {\sqrt{2\log(1 + 2/\epsilon_x)} \over \pi \Delta}$;\ \ for $d\ge 2$, pick
  $ R\ge {\sqrt{2\log(k/\epsilon_x)} \over \pi \Delta}$.

  Assume the bounded measurement noise model as in \eqref{eq:measure-noise} and  that \qq{ $\epsilon_z\le
  {\Delta\delta_{v}w_{min}^2\over 100\sqrt{d}k^5} \lt({1- 2\epsilon_x\over 1+2\epsilon_x}\rt)^{2.5} $}.

  With probability at least $(1-\delta_s)$ over the random sampling of $\mc S$, and with probability
  at least $(1- \delta_{v})$ over the random projections in Algorithm~\ref{alg:tensordecomp}, the
  proposed Algorithm~\ref{alg:gen} returns an estimation of the point source signal $\wh x(t) =
  \sum_{j=1}^{k} \wh w_j \wh\delta_{\mu^{(j)}}$ with accuracy:
  \begin{align*}
    & \min_{\pi}\max\lt\{ \|\wh \mu^{(j)} - \mu^{(\pi(j))} \|_2: j\in[k]\rt\} \le
    C{ \sqrt{d}k^5 \over \Delta \delta_v} {w_{max} \over w_{min}^2} \lt({1+ 2\epsilon_x\over
      1-2\epsilon_x}\rt)^{2.5} \epsilon_z,
   \end{align*}
   where the $\min$ is over permutations $\pi$ on $[k]$.  Moreover, the proposed algorithm has time
   complexity in the order of $O((m')^3)$.
\end{theorem}

\begin{proof}(of Theorem~\ref{thm:main-thm})
  The algorithm is correct if the tensor decomposition in Step 2 is unique, and achieves stable
  recovery if the tensor decomposition is stable.  By the stability Lemma of tensor decomposition
  (Lemma~\ref{lem:stable-jennrich}), this is guaranteed if we can bound the condition number of
  $V_{S'}$.
  It follows from Lemma~\ref{lem:S-step2} that the condition number of $V_{S'}$ is at most
  $\sqrt{1+\sqrt{k}}$ times of $\cond_2(V_S)$. By the main technical lemma
  (Lemma~\ref{lem:main-tech}) we know that with the random sampling set $\mc S$ of size $m$, the
  condition number $\cond_2(V_S)$ is upper bounded by a constant.
  Thus we can bound the distance between $V_{S'}$ and the estimation $\wh V_{S'}$ according to
  \eqref{eq:stable-jennrich-A}.

  Since we adopt Jennrich's algorithm for the low rank tensor decomposition, the overall computation
  complexity is roughly the complexity of SVD of a matrix of size $m'\times m'$, namely in the order
  of $O((m')^3)$.
\end{proof}

The next lemma shows that essentially, with overwhelming probability, all the frequencies taken
concentrate within the hyper-cube with cutoff frequency $R'$ on each coordinate, where $R'$ is
comparable to $R$,
\begin{lemma}[The cutoff frequency]
  \label{rem:cutoff}
  For $d>1$, with high probability, all of the \qq{$2(m')^2$} sampling frequencies in \qq{$\mc
    S'\oplus \mc S'\oplus \{v^{(1)}, v^{(2)}\}$} satisfy that $ \|s^{(j_1)} + s^{(j_2)} + v^{(j_3)}
  \|_{\infty} \le R',\quad \fa j_1,j_2\in[m], j_3\in[2], $ where the per-coordinate cutoff frequency
  is given by $R' = O(R \sqrt{\log{m d}})$.

  For $d=1$ case, the cutoff frequency $R'$ can be made to be in the order of $R'= O(1/\Delta)$.
\end{lemma}
\begin{proof}
For $d>1$ case,  with straightforward union bound over the $m'= O(k^2)$ samples each of which has
$d$ coordinates, one can show that the cutoff frequency is in the order of $R\sqrt{\log(kd)}$,
where $R$ is in the order of $ {\sqrt{\log(k)} \over  \Delta}$ as shown in
Theorem~\ref{thm:main-thm}.

For $d=1$ case, we bound the cutoff frequency with slightly more careful analysis.  Instead of
Gaussian random samples, consider uniform samples from the interval $[-R',R']$. We can modify the
proof of Lemma~\ref{lem:p2-ingham} and show that if $R' \ge 1/(\Delta(1+\epsilon_x))$:
\begin{align*}
  \sum_{j'\neq j} |Y_{j,j'}|& = \sum_{j'\neq j} {1\over 2R'}\int_{-R',R'} e^{i\pi(\mu^{j'}-\mu^{(j)})s }
  = \sum_{j'\neq j} {\sin(\pi |\mu^{(j')}-\mu^{(j)}|R') \over \pi |\mu^{(j')}-\mu^{(j)}|R'}
  \\&\le \sum_{l=1}^k {\sin(l \pi \Delta R') \over (l \pi \Delta R')}
  \le {\sin(\pi \Delta R')/  ( \pi \Delta R') \over 1 - \sin(\pi \Delta R')/  ( \pi \Delta R')}  \le \epsilon_x
\end{align*}
where the second last inequality uses the inequality that ${\sin(a+b)\over a+b} \le {\sin(a)\over a}
{\sin(b)\over b}$.
\end{proof}

\begin{remark}[Failure probability]
  Overall, the failure probability consists of two pieces: $\delta_{v}$ for random projection of $v$,
  and $\delta_{s}$ for random sampling to ensure the bounded condition number of $V_S$. This may be
  boosed to arbitrarily high probability through repetition.
\end{remark}

\subsection{Key Lemmas}


\sk{also should check that we like the stability argument proof}\qt{done..}

{\bf Stability of tensor decomposition:}
In this paragraph, we give a brief description and the stability guarantee of the well-known
Jennrich's algorithm (\cite{harshman1970foundations,leurgans1993decomposition}) for low rank 3rd
order tensor decomposition.  We only state it for the symmetric tensors as appeared in the proposed
algorithm.

Consider a tensor $F = V\ot V\ot (V_2D_{w})\in\C^{m\times m\times 3}$ where the factor $V$ has full
column rank $k$. Then the decomposition is unique up to column permutation and rescaling, and
Algorithm~\ref{alg:tensordecomp} finds the factors efficiently.
Moreover, the eigen-decomposition is stable if the factor $V$ is well-conditioned and the
eigenvalues of $F_aF_b^\dag$ are well separated.
\begin{lemma}[Stability of Jennrich's algorithm]
  \label{lem:stable-jennrich}
  Consider the 3rd order tensor $F= V\ot V\ot( V_2D_w)\in\C^{m\times m \times 3}$ of rank $k\le
  m$, constructed as in Step 1 in Algorithm 1.
  %

  Given a tensor $\wt F$ that is element-wise close to $F$, namely for all $n_1,n_2,n_3\in[m]$,
  $\big|\wt F_{n_1,n_2,n_3}- F_{n_1,n_2,n_3}\big|\le \epsilon_z $, and assume that the noise is small $
  \epsilon_z \le {\Delta \delta_v w_{min}^2 \over 100 \sqrt{dk} w_{max} \cond_2(V)^5
  } $.
  Use $\wt F$ as the input to Algorithm~\ref{alg:tensordecomp}. With probability at least
  $(1-\delta_{v})$ over the random projections $v^{(1)}$ and $v^{(2)}$, we can bound the distance
  between columns of the output $\wh V$ and that of $V$ by:
  \begin{align}
    \label{eq:stable-jennrich-A}
    \min_{\pi} \max_{j} \lt\{\|\wh V_j - V_{\pi(j)}\|_2:j\in[k]\rt\} &\le C{ \sqrt{d}k^2 \over \Delta
      \delta_v} {w_{max} \over w_{min}^2} \cond_2(V)^5 \epsilon_z,
      \end{align}
      where $C$ is a universal constant.
\end{lemma}

\begin{proof} (of Lemma~\ref{lem:stable-jennrich}) The proof is mostly based on the arguments in
  \cite{mossel2005learning, anandkumar2012method}, we still show the clean arguments here for our
  case.

  We first introduce some notations for the exact case.
  Define $D_1 = \diag([V_2]_{1,:}D_w)$ and $D_2 = \diag([V_2]_{2,:}D_w)$.
  Recall that the symmetric matrix $F_1 = F(I,I,e_1) = VD_1 V^\top$.  Consider its SVD $F_1 = P
  \Lambda P^\top$. Denote $U = P^\top V \in\C^{k\times k}$. Define the whitened rank-$k$ tensor $$E
  = F(P, P, I) = (P^\top V) \ot( P^\top V) \ot (V_2 D_w ) = U \ot U\ot (V_2D_w)\in \C^{k\times
    k\times 3}.$$ Denote the two slices of the tensor $E$ by $E_1 = E(I,I,e_1) = U D_1 U^\top$ and
  $E_2 = E(I,I,e_2) = U D_2 U^\top$.
  Define $M = E_1E_2^{-1}$, and its eigen decomposition is given by $M = U DU^{-1}$, where $D =
  D_1D_2^{-1}$. Note that in the exact case, $D$ is given by:
  \begin{align*}
    D= \diag(e^{i\pi<\mu^{(j)}, v^{(1)} -v^{(2)}>}:j\in[k])
  \end{align*}
  Note that $|D_{j,j}|=1$ for all $j$.
  Define the minimal separation of the diagonal entries in $D$ to be:
  \begin{align*} sep(D) = \min \{\min_{j\neq j'} |D_{j,j} - D_{j',j'}|\},
  \end{align*}

  1.  We first apply perturbation bounds to show that the noise in $\wt F$ propagates  the estimates
  $\wh P$ and $\wh E$ in a mild way when the condition number of $V$ is bounded by a constant.
  \begin{proof}
    Apply Wedin's matrix perturbation bound, we have:
    \begin{align*}
      \|\wh P - P\|_2 \le{ \|\wt F_1 - F_1\|_2 \over \sigma_{min}(F_1)}
      \le{ \epsilon_z \sqrt{m}\over w_{min}\sigma_{min}(V)^2 }
    \end{align*}
    And then for the two slices of $\wh E=\wt F(\wh P,\wh P,I)$, namely $\wh E_i = E_i + Z_i$ for
    $i=1,2$, we can bound the distance between estimates and the exact case, namely $Z_i = \wh
    P^\top \wt F_i \wh P - P^\top F_i P $, by:
    \begin{align*}
      \| Z_i\| \le 8\|F_i\|\|P\|\|\wh P - P\| + 4\|P\|^2\|\wt F_i- F_i\| \le 16 {w_{max}
        \over w_{min}} \cond_2(V)^2 \epsilon_z\sqrt{m}
    \end{align*}

  \end{proof}

  2. Then, recall that $M = E_1E_2^{-1} = U D U^{-1}$.  Note that $$\wh M = (E_1 + Z_1) (E_2 +
  Z_2)^{-1} = E_1 E_2^{-1}(I - Z_2(I+E_2^{-1} Z_2)^{-1} E_2^{-1} ) + Z_1 E_2^{-1}.$$  Let $H$ and
  $G$ denote the
  perturbation matrices:
  \begin{align*}
    H = - Z_2(I+E_2^{-1} Z_2)^{-1} E_2^{-1}, \quad G = Z_1 E_2^{-1}.
  \end{align*}
  In the following claim, we show that given $\wh M = \wh E_1 \wh E_2^{-1} = M (I + H) + G$ for some
  small perturbation matrix $H$ and $G$, if the perturbation $\|H\|$ and $\|G\|$ are small enough
  and that $sep(D)$ is large enough, the eigen decomposition $\wh M = \wh U \wh D \wh U^{-1}$ is
  close to that of $M$.

  \begin{claim}
    If \qq{$\|MH + G\|\le {sep(D)\over 2\sqrt{k} \cond_2(U)}$}, then the eigenvalues of $\wh M$ are
    distinct and we can bound the columns of $\wh U$ and $U$ by:
    \begin{align*}
      \min_{\pi}\max_{j} \|\wh U_j - U_{\pi(j)}\|_2 \le  3
      {\sigma_{max}(H)\sigma_{max}(D)+\sigma_{max}(G) \over \sigma_{min}(U) sep(D)} \|\wh U_j\|_2\|V_j\|_2.
    \end{align*}
  \end{claim}
  \begin{proof}
    Let $\lambda_j$ and $U_j$ for $j\in[k]$ denote the eigenvalue and corresponding eigenvectors of
    $M$.
    If $\|MH + G\|\le {sep(D)\over 2\sqrt{k} \cond_2(U)}$, we can
    bound
    \begin{align*}
      \| \wh M - M \| = \|U^{-1}(M + (MH+G)) U - D\| = \|U^{-1}(MH+G) U\| \le sep(D)/2\sqrt{k},
    \end{align*}
    thus apply Gershgorin's disk theorem, we have $ |\wh \lambda_j - \lambda_j| \le
    \|[U^{-1}(MH+G)U]_j\|_{1} \le \sqrt{k} \| [U^{-1}(MH+G)U]_j\|_2 \le sep(D)/2$.  Therefore, the
    eigenvalues are distinct and we have
    \begin{align}
      \label{eq:lambda-ineq}
      |\wh \lambda_j - \lambda_{j'}|\ge | \lambda_j - \lambda_{j'}| - |\wh \lambda_j -
      \lambda_{j}|\ge {1\over 2} | \lambda_j - \lambda_{j'}| \ge {1\over 2} sep(D).
    \end{align}

    Note that $\{U_{j'}\}$ and $\{\wh U_{j}\}$ define two sets of basis vectors, thus we can write
    $\wh U_j = \sum_{j'}c_{j'} U_{j'}$ (with the correct permutation for columns of $\wh U_j$ and
    $U_j$) for some coefficients $\sum_{j'}c_{j'}^2=1$.  Apply first order Taylor expansion of
    eigenvector definition we have:
    \begin{align*}
      \wh \lambda_j \wh U_j = \wh M \wh U_j = (M + (MH+G)) \sum_{j'}c_{j'} U_{j'} = \sum_{j'}
      \lambda_{j'} c_{j'} U_{j'} + (MH+G) \wh U_j.
    \end{align*}
    Since we also have $ \wh \lambda_j \wh U_j = \sum_{j'} \wh \lambda_{j} c_{j'} U_{j'}$, we can
    write $\sum_{j'} (\wh \lambda_{j} - \lambda_{j'}) c_{j'} U_{j'} = (MH+G)\wh U_{j}$, and we can
    solve for the coefficients $c_{j'}$'s from the linear system as $[(\wh \lambda_j -
    \lambda_{j'})c_{j'}:j'\in[k] ]= U^{-1} (MH+G)\wh U_{j}$.
    Finally plug in the inequality in \eqref{eq:lambda-ineq} we have that for any $j$:
    \begin{align*}
      \|\wh U_j - U_j\|_2^2 &= \sum_{j'\neq j} c_{j'}^2 \|U_{j'}\|_2^2 + (c_{j} - 1)^2 \|U_{j}\|_2^2
      \\
      &\le 2 \sum_{j'\neq j} c_{j'}^2 \|V_{j'}\|_2^2
      \\
      &\le 8 {\|U^{-1} (MH+G) \wh U_{j}\|_2^2 \over sep(D)^2}
      \\
      &\le 8 { (\sigma_{max}( D) \sigma_{max}(H) + \sigma_{max}(G))^2 \over
        \sigma_{min}(U)^2sep(D)^2}\|\wh U_j\|_2^2 \|V_{j}\|_2^2
    \end{align*}
  \end{proof}

  3. Note that in the above bound for $\|\wh U_j - U_j\|$, we can bound the perturbation matrices
  $H$ and $G$ by:
    \begin{align*}
      &\sigma_{max} (H) \le {\|Z_2\| \over (1- \sigma_{max}(E_2^{-1}Z_2)) \sigma_{min}(E_2) } \le
      {\|Z_2\| \over \sigma_{min}(E_2) - \|Z_2\|} \le { \|Z_2\|\over \sigma_{min}(U)^2 \sigma_{min}(D_{2}) -
        \|Z_2\|},
      \\
      &\sigma_{max}(G) \le {\sigma_{max}(Z_1) \over \sigma_{min}(E_2)} \le { \|Z_2\| \over
        \sigma_{min}(U)^2 \sigma_{min}(D_{2})},
    \end{align*}
    Note that $\sigma_{min}(D_2) \ge w_{min} $ and $\sigma_{\max}(D) = 1$ by definition.  In the
    following claim, we apply anti-concentration bound to show that with high probability
    $sep(D)$ is large.

    \begin{claim}
      For any $\delta_{v}\in(0,1)$, with probability at least $1-\delta_v$, we can bound $sep(D)$
      by:
      \begin{align*}
        sep(D) \ge {\Delta\delta_v\over \sqrt{d} k^2}.
      \end{align*}

  \end{claim}
  \begin{proof}
    Denote $v = v^{(1)} - v^{(2)}$, and note that $\|v\| \le \sqrt{2}$. In the regime we concern,
    for any pair $j\neq j'$, we have $|e^{i\pi<\mu^{(j)}, v>} - e^{i\pi<\mu^{(j')}, v>} | \le
    |<\mu^{(j)} - \mu^{(j')},v>|$. Apply Lemma~\ref{lem:anti-concentration}, we have that for
    $\delta\in(0,1)$,
      \begin{align*}
        \mbb P( |<\mu^{(j)} - \mu^{(j')},v>| \le \|\mu^{(j)} - \mu^{(j')}\| {\delta\over
          \sqrt{d}})\le \delta.
      \end{align*}
      Take a union bound over all pairs of $j\neq j'$, we have that
      \begin{align*}
        \mbb P\lt( \tx{for some} j\neq j', |<\mu^{(j)} - \mu^{(j')},v>| \le \|\mu^{(j)} -
        \mu^{(j')}\| {\delta\over \sqrt{d} k^2}\rt)\le k^2 {\delta\over k^2} = \delta.
      \end{align*}
      Recall that $\Delta = \min_{j\neq j'} \|\mu^{(j)}-\mu^{(j')}\|$.
  \end{proof}

  4.  Recall that $U = P^\top V$. Note that since $P$ has orthonormal columns, we have
  $\sigma_{min}(U) = \sigma_{min}(V)$ and $\|U_i\|\le \|V_i\|=\sqrt{m}$.

  Finally we apply perturbation bound to the estimates $\wh V_i = \wh P\wh U_i$ and conclude with
  the above inequalities:
  \begin{align*}
    \|\wh V_i - V_i\|&\le 2(\| \wh P - P \| \|U_i\| + \|P\|\|\wh U_i - U_i\|)
    \\
    &\le 2\lt( { \epsilon_z \sqrt{m}\over w_{min}\sigma_{min}(V)^2 } + 3
    {\sigma_{max}(H)\sigma_{max}(D)+\sigma_{max}(G) \over \sigma_{min}(U) sep(D)} \|V_i\|\rt)\|V_i\|
    \\
    &\le 2\lt( { \epsilon_z \sqrt{m}\over w_{min}\sigma_{min}(V)^2 } + 6 { \|Z_2\| \|V_i\|\over
      (\sigma_{min}(V)^2 \sigma_{min}(D_{2}) - \|Z_2\| ) \sigma_{min}(V) sep(D)} \rt)\|V_i\|
    \\
    &\le C({ \sqrt{d} k^2m \over \Delta \delta_v}  {w_{max} \cond_2(V)^2
      \over w_{min}^2\sigma_{min}(V)^3}) \|V_i\|\epsilon_z,
    \end{align*}
for some universal constant $C$.    Note that  the last inequality used the assumption that $\epsilon_z$ is small enough.
\end{proof}

{\bf Condition number of $V_{S'}$:}
The following lemma is helpful:
\begin{lemma}
\label{lem:S-step2}

Let $V_{S'}\in\C^{(m+d+1)\times k}$ be the factor as defined in \eqref{eq:Vs-p}. Recall that $V_{S'} =
[V_S; V_d; 1]$, where $V_d$ is defined in \eqref{eq:Vd}, and $V_S$ is the characteristic matrix defined
in \eqref{eq:Vs-general}.

We can bound the condition number of $V_{S'}$ by
\begin{align}
  \label{eq:cond2-vs-vsp}
    \cond_{2}(V_{S'})\le \sqrt{1+\sqrt{k}}\cond_2(V_S).
  \end{align}
\end{lemma}

\begin{proof}(of Lemma~\ref{lem:S-step2})
  By definition, there exist some constants $\lambda$ and $\lambda'$ such that $\cond_{2}(V_S) =
  \lambda'/\lambda$, and for all $w\in\mc P_{1,2}^k$, we have $ \lambda \le \|V_{S} w\| \le \lambda'$.
  Note that each element of the factor $V_{S'}$ lies on the unit circle in the complex plane, then
  we have:
  \begin{align*}
    \lambda^2 \le \|V_{S} w\|_2^2 \le \|V_{S'}w\|_2^2 \le (\lambda')^2 + \sqrt{k}d.
  \end{align*}
  We can bound the condition number of $V_{S'}$ by:
  \begin{align*}
    \cond_{2}(V_{S'}) \le \sqrt{(\lambda')^2 + \sqrt{k}d \over \lambda^2} = \sqrt{1+{\sqrt{k}d \over
        (\lambda')^2} }\cond_2(V_S) \le \sqrt{1+{\sqrt{k}} }\cond_2(V_S),
  \end{align*}
  where the last inequality is because that $\max_{w} \|V_{S} w\|_2^2 \ge \|V_{S} e_1\|_2^2 = d$, we
  have $(\lambda')^2 \ge d$.

\end{proof}

{\bf Condition number of the characteristic matrix $V_S$:}
Therefore, the stability analysis of the proposed algorithm boils down to understanding the relation
between the random sampling set $\mc S$ and the condition number of the characteristic matrix $V_S$.
This is analyzed in Lemma~\ref{lem:main-tech} (main technical lemma).

\begin{lemma}
  \label{lem:p2-ingham}
  For any fixed number $\epsilon_x\in(0,1/2)$.  Consider a Gaussian vector $s$ with distribution
  $\mc N(0, R^2 I_{d\times d})$, where $R\ge {\sqrt{2\log(k/\epsilon_x)} \over \pi \Delta}$
  for $d\ge 2$, and $R\ge {\sqrt{2\log(1 + 2/\epsilon_x)} \over \pi \Delta}$ for $d=1$.
  Define the Hermitian random matrix $X_s \in\C^{k\times
    k}_{herm}$ to be
  \begin{align}
    \label{eq:def-X}
    X_s =\lt[
    \begin{array}[c]{c}
      e^{-i\pi<\mu^{(1)},s>}\\ e^{-i\pi<\mu^{(2)},s>} \\ \vdots \\e^{-i\pi<\mu^{(k)},s>}
    \end{array}
    \rt] \lt[ e^{i\pi<\mu^{(1)},s>}, e^{i\pi<\mu^{(2)},s>}, \dots e^{i\pi<\mu^{(k)},s>}\rt].
  \end{align}
  We can bound the spectrum of $\mbb E_{s}[X_s]$ by:
  \begin{align}
    \label{eq:X-bound}
    (1-\epsilon_x)I_{k\times k} \preceq \mbb E_s[X_s] \preceq (1+\epsilon_x)I_{k\times k}.
  \end{align}
\end{lemma}

\begin{proof}(of Lemma~\ref{lem:p2-ingham})
  Denote $Y = \mbb E_s[X_s]$.  Note that $Y_{j,j}=1$ for all diagonal entries.
  For $d=1$ case, the point sources all lie on the interval $[-1,1]$, we can bound the summation of
  the off diagonal entries in the matrix $Y$ by:
  \begin{align*}
    \sum_{j'\neq j} |Y_{j,j'}|& = \mbb E_s[e^{i\pi<\mu^{(j')}-\mu^{(j)}, s>}]
    \\
    &= \sum_{j'\neq j} e^{-{1\over 2} \pi^2 \|\mu^{(j)}-\mu^{(j')}\|_2^2
      R^2}
    \\
    &\le 2( e^{-{1\over 2}( \pi \Delta R)^2} + e^{-{1\over 2}( \pi (2 \Delta) R)^2} +
    \dots + e^{-{1\over 2}( \pi (k/2)\Delta R)^2})
    \\
    &\le {2  e^{-{1\over 2}( \pi \Delta R)^2} / (1- e^{-{1\over 2}( \pi \Delta R)^2}) }
    \\
    &\le \epsilon_x.
  \end{align*}
  For $d\ge 2$ case, we simply bound each off-diagonal entries by:
  \begin{align*}
    Y_{j,j'} &
    =  e^{-{1\over 2} \pi^2 \|\mu^{(j)}-\mu^{(j')}\|_2^2 R^2}
    \le e^{-{1\over 2} \pi^2 \Delta^2 R^2} \le \epsilon_x/k.
  \end{align*}
  Apply Lemma~\ref{lem:geshgorin-disk} (Gershgorin's Disk Theorem) and we know that all the
  eigenvalues of $Y$ are bounded by $1\pm \epsilon_x$.
\end{proof}

\begin{lemma}[Main technical lemma]
  \label{lem:main-tech}
  In the same setting of Lemma~\ref{lem:p2-ingham},
  Let $\mc S= \{s^{(1)},\dots, s^{(m)}\}$ be $m$
  independent samples of the Gaussian vector $s$.
  For $ m\ge {k\over \epsilon_x}\sqrt{8\log{k\over \delta_s}}$, with probability at least
  $1-\delta_s$ over the random sampling, the condition number of the
  factor $V_S$ is bounded by:
  \begin{align}
   \label{eq:cond-Vs}
   \cond_2(V_S) \le \sqrt{1+ 2\epsilon_x\over 1-2\epsilon_x}.
 \end{align}
\end{lemma}

\begin{proof}(of Lemma~\ref{lem:main-tech})
  Let $\{X^{(1)},\dots, X^{(m)}\}$ denote the i.i.d.  samples of the random matrix $X_s$ defined in
  \eqref{eq:def-X}, with $s$ evaluated at the i.i.d. random samples in $\mc S$.
  Note that we have
  \begin{align*}
    \|V_Sw\|_2^2 = w^\top V_S^* V_S w = w^\top \lt({1\over m} \sum_{i=1}^m X^{(i)} \rt) w.
  \end{align*}
  By definition of condition number, to show that $ \cond_2(V_S) \le \sqrt{1+ 2\epsilon_x\over
    1-2\epsilon_x}$, it suffices to show that $$(1-2\epsilon_x) I_{k\times k} \preceq \lt({1\over m}
  \sum_{i=1}^m X^{(i)} \rt) \preceq (1+ 2\epsilon_x) I_{k\times k}.$$

  By Lemma~\ref{lem:p2-ingham}, the spectrum of $\mbb E_s[X_s]$ lies in $(1-\epsilon_x,
  1+\epsilon_x)$.  Here we only need to show that the spectrum of the  sample mean $\lt({1\over m} \sum_{i=1}^m X^{(i)}
  \rt) $ is close to the spectrum of the  expectation $\mbb E_s[X_s]$.
  Since each element of the random matrix $X_s\in\C^{k\times k}$ lies on the unit circle in the
  complex plane, we have $X_s^2\preceq k^2 I$ almost surely. Therefore we can apply
  Lemma~\ref{lem:matrix-hoeffding} (Matrix Hoeffding) to show that for $m> {k\over
    \epsilon_x}\sqrt{8\log{k\over \delta_s}}$, with probability at least $1-\delta_s$, it holds that $\|
  {1\over m} \sum_{i=1}^m X^{(i)} - \mbb E_s[ X_s ]\|_2 \le \epsilon_x$.
\end{proof}

\section{Discussions}

\subsection{Numerical results}
We empirically demonstrate the performance of the proposed super-resolution algorithm in this section.

First, we look at a simple instance with dimension $d=2$ and the minimal separation $\Delta = 0.05$.
\qq{Our perturbation analysis of the stability result limits to small noise, i.e. $\epsilon_z$ is
  inverse polynomially small in the dimensions, and the number of measurements $m$ needs to be
  polynomially large in the dimensions.
  However, we believe these are only the artifact of the crude analysis, instead of being intrinsic
  to the approach.}
In the following numerical example, we examine a typical instance of 8 randomly generated 2-D point
sources.
The minimal separation $\Delta$ is set to be 0.01, and the weights are uniformly distributed in $[0.1,1.1]$
The measurement noise level $\epsilon_z$ is set to be 0.1, and we take only $2178$ noisy measurements ($\ll 1/\Delta^2$).
Figure~\ref{fig:noisy recovery in 2D} shows reasonably good recovery result.
\begin{figure}[H]
  \centering{
    \includegraphics[width = 0.5\textwidth]{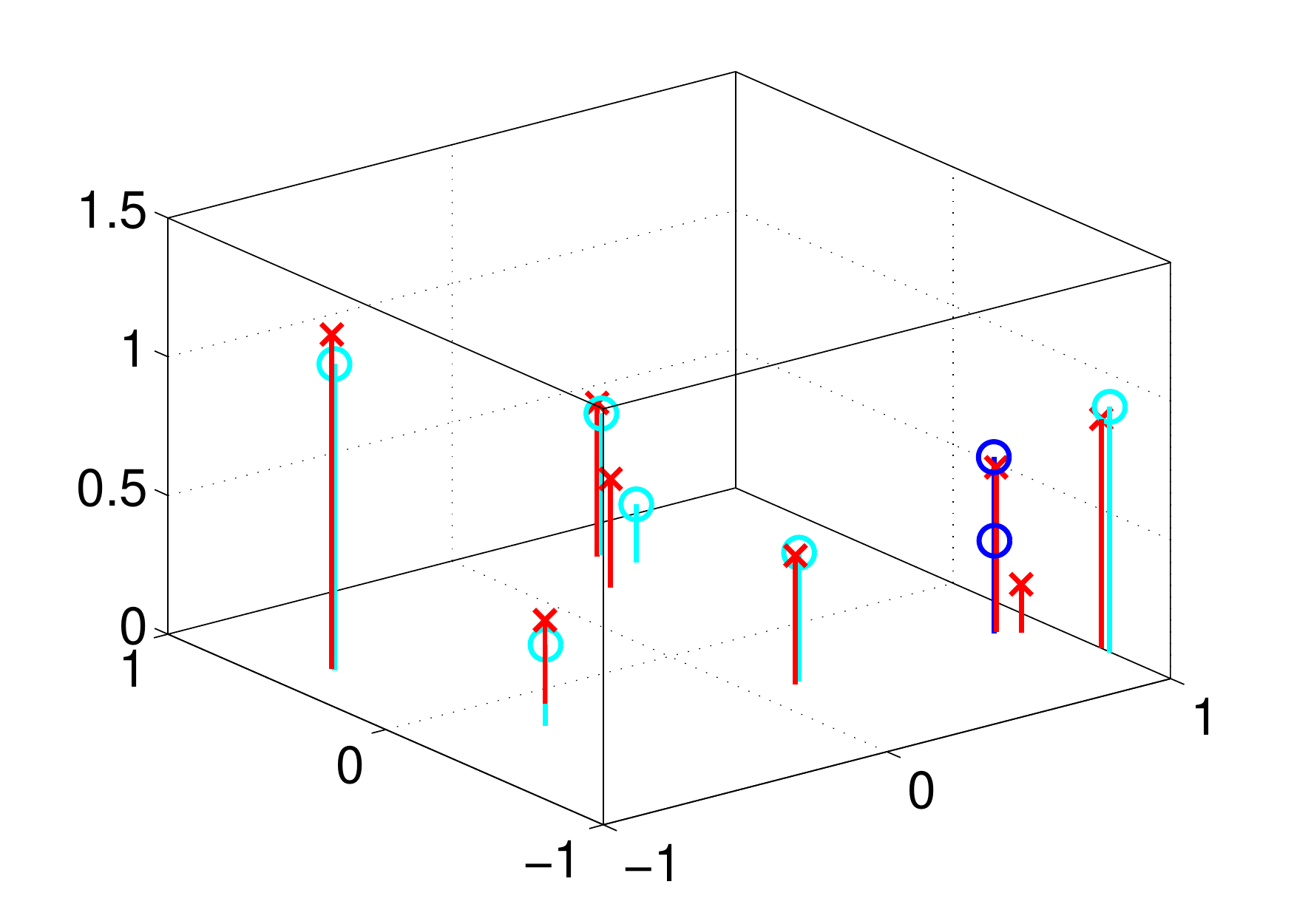}
  }
  \caption{The xy plane shows the coordinates of the point sources: true point sources (cyan), the
    two closest points (blue), and the estimated points (red); the z axis shows the corresponding
    mixing weights.
    Dimension $d =2$, number of point sources $k=8$, minimal separation $\Delta=0.05$ and the
    measurement noise level $\epsilon_z = 0.1$; we set the cutoff frequency $R = 200$ (in the same order
    as $1/\Delta$), take $2178$ random measurements ($\ll 1/\Delta^2$). }
  \label{fig:noisy recovery in 2D}
\end{figure}

Next, we examine the phase transition properties implied by the main theorem.

Figure~\ref{fig:phase-transition1} shows the dependency between the cutoff frequency and the minimal
separation.
 For each fixed pair of the minimal separation and the cutoff frequency $(\Delta, R)$, we randomly
 generate $k=8$ point sources in $4$-dimensional space while maintaining the same
 minimal separation. The weights are uniformly distributed in $[0.1,1.1]$.
The recovery is considered successful if the error $\sum_{j\in[k]} \sqrt{\|\wh \mu^{(j)} -
  \mu^{(j)}\|_2^2} \le 0.1$ (on average it tolerates around $4\%$ error per coordinate per
point source).
This process is repeated 50 times and the rate of success was recorded.
Figure~\ref{fig:phase-transition1} plots the success rate in gray-scale, where 0 is black and 1 is
white.

We observe that there is a sharp phase transition characterized by a linear relation
between the cutoff frequency and the inverse of minimal separation, which is implied by
Theorem~\ref{thm:main-thm}.

\begin{figure}[H]
  \centering{
    \includegraphics[width = 0.6\textwidth]{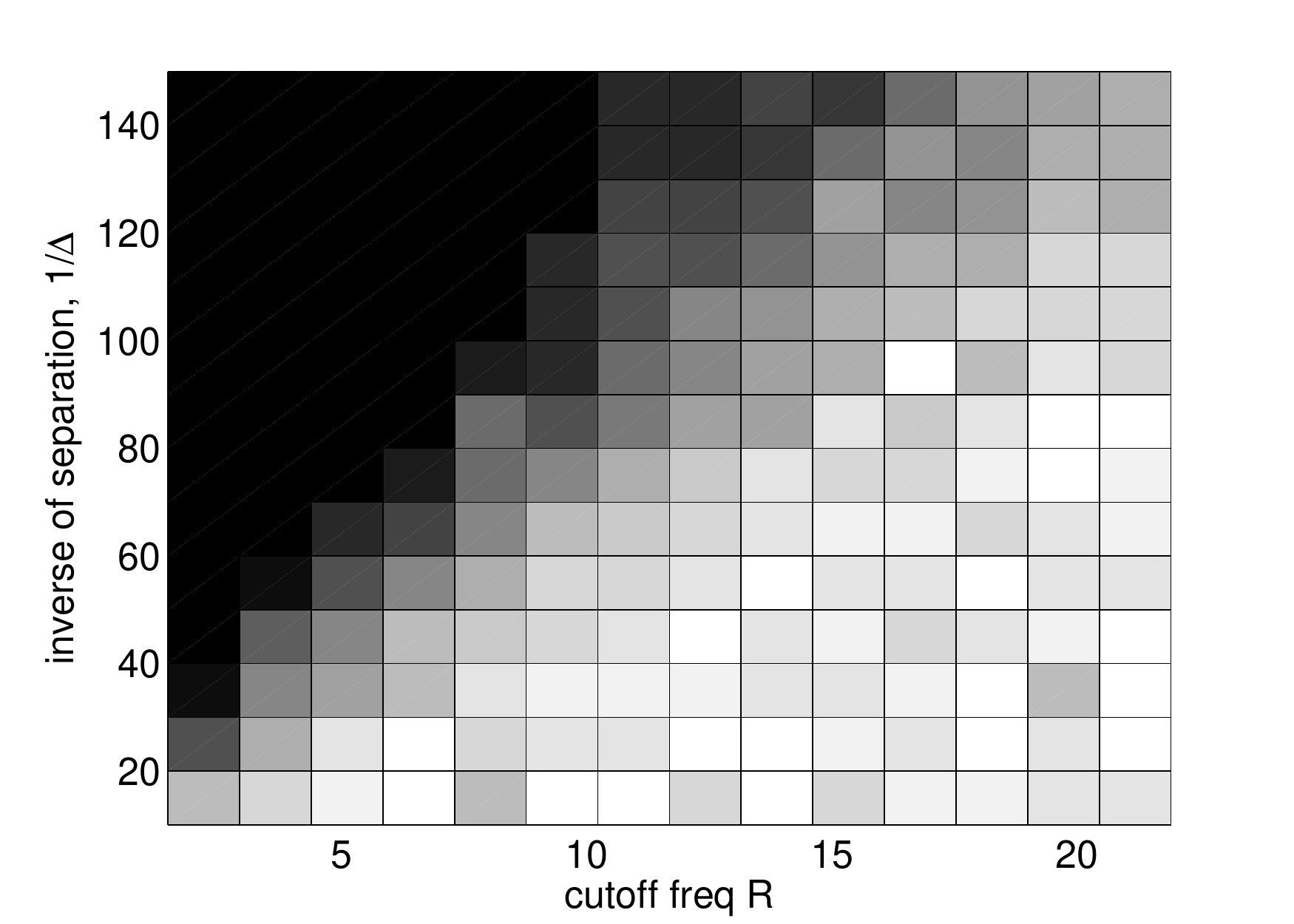}
  }
  \caption{Fix dimension $d =4$, number of point sources $k=8$,  number of
    measurements $m = k^2$, and the measurement noise level $\epsilon_z = 0.02$.
    We vary the minimal separation such that $\Delta$ ranges from 0.005 to 0.1, and we
    vary the cutoff frequency $R$ from 0 to 25.
    For each pair of $({1\over \Delta}, R)$ we
    randomly generate $k$ point sources and run the proposed algorithm to recover the
    point sources.  The recovery is considered successful
    if the error $\sum_{j\in[k]} \sqrt{\|\wh \mu^{(j)} - \mu^{(j)}\|_2^2} \le 0.1$.
    This process is repeated 50 times and the rate of success was recorded.  }
  \label{fig:phase-transition1}
\end{figure}

In a similar setup, we examine the success rate while varying the minimal separation
$\Delta$ and  the number of measurement $m$.

In Figure~\ref{fig:phase-transition2}, we observe that there is a threshold of $m$ below
which the number of measurements is too small to achieve stable recovery; when $m$ is
above the threshold, the success rate increases with the number of measurements as the
algorithm becomes more stable.\    %
However, note that given the appropriately chosen cutoff frequency $R$, the number of measurements
required does not depend on the minimal separation, and thus the computation complexity does not
depend on the minimal separation neither.

\begin{figure}[h!]
  \centering{
    \includegraphics[width = 0.6\textwidth]{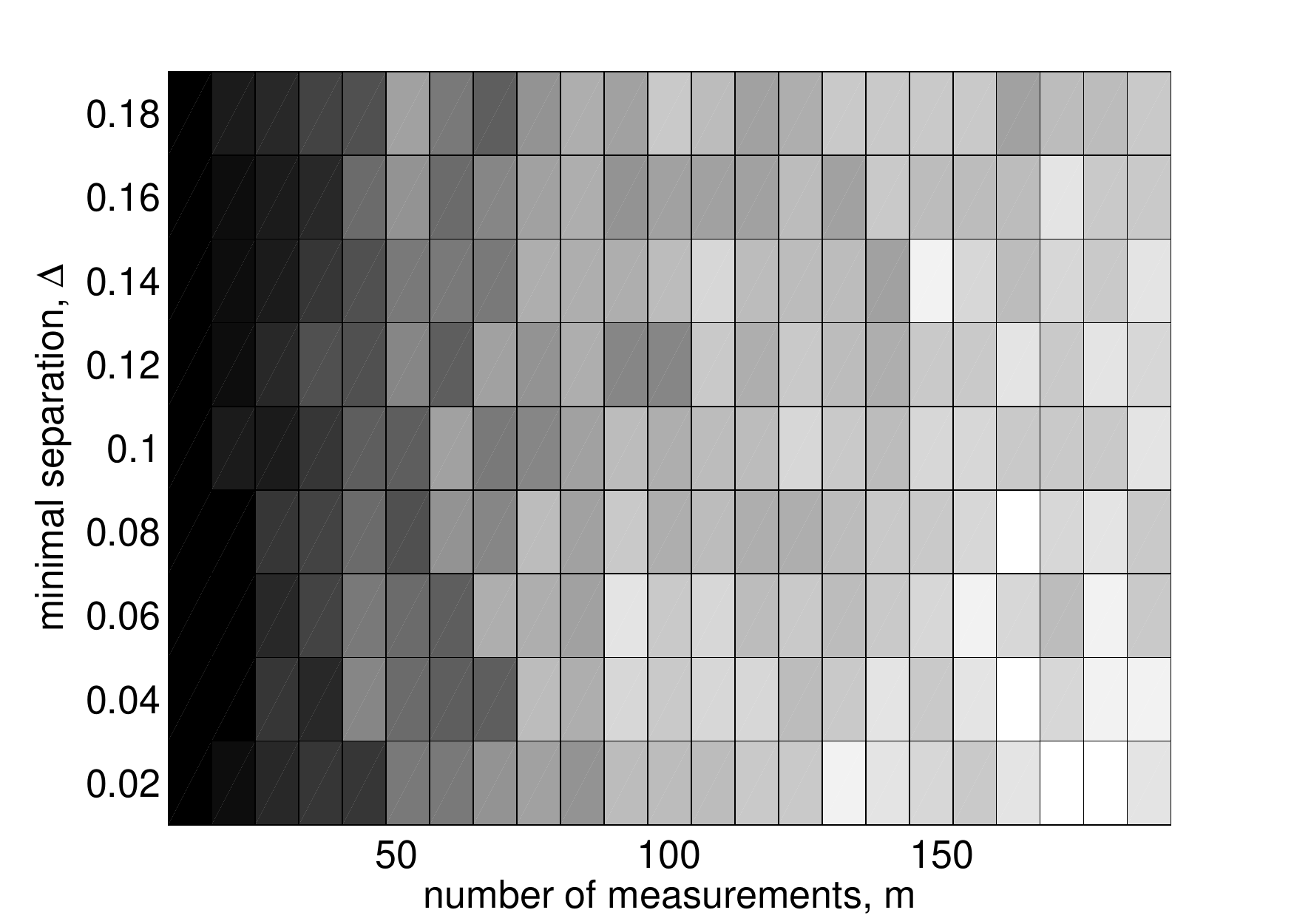}
  }
  \caption{Fix dimension $d =4$, number of point sources $k=8$, and the measurement noise
    level $\epsilon_z = 0.03$.  %
    We vary the minimal separation such that $\Delta$ ranges from 0.01 to 0.2, and we use
    the corresponding cutoff frequency $R= {0.26\over \Delta}$ . We also vary the number
    of measurements $m$ from 4 to 64.
    For each pair of $(\Delta, m)$ we randomly generate $k$ point sources and run the
    proposed algorithm to recover the point sources. The recovery is considered successful
    if the error $\sum_{j\in[k]} \sqrt{\|\wh \mu^{(j)} - \mu^{(j)}\|_2^2} \le 0.1$.
    This process is repeated 50 times and the rate of success was recorded. }
  \label{fig:phase-transition2}
\end{figure}

\subsection{ Connection with learning GMMs}
One reason  we are interested in the scaling of the algorithm with respect to the dimension
$d$ is that it naturally leads to an algorithm for learning Gaussian mixture models
(GMMs).

We first state the problem: given a number of $N$ i.i.d. samples coming from a random one
out of $k$ Gaussian distributions in $d$ dimensional space, the learning problem asks to estimate
the means and the covariance matrices of these Gaussian components, as well as the mixing
weights.
We denote the parameters by $\{(w_j,\mu^{(j)}, \Sigma^{(j)})\}_{i\in[k]}$ where the mean vectors
$\mu^{(j)}\in[-1,+1]^d$, the covariance matrices $\Sigma^{(j)}\in\R^{d\times d}$ and the mixing
weights $w_j\in\R_+$.
Learning mixture of Gaussians is a fundamental problem in statistics and machine learning, whose
study dates back to Pearson\cite{pearson1894contributions} in the 1900s, and later arise in numerous
areas of applications.

In this brief discussion, we only consider the case where the components are spherical
Gaussians with common covariance matrices, namely $\Sigma^{(j)} = \sigma^2 I_{d\times d}$
for all $j$.
Moreover,  we define the separation $\Delta_G$ by:
\begin{align*}
  \Delta_G = {\min_{j\neq j'}\|\mu^{(j)} - \mu^{(j')}\|_2 \over \sigma},
\end{align*}
and we will focus on the well-separated case where $\Delta_G$ is sufficiently large.
This class of well-separated GMMs is often used in data clustering.

By the law of large numbers, for large $d$, the probability mass of a $d$-dimensional Gaussian
distribution tightly concentrates within a thin shell with a $\sqrt{d} \sigma$ distance from the
mean vector. This concentration of distance leads to a line of works of provably learning GMMs in
the well-separated case, started by the seminal work of Dasgupta\cite{dasgupta1999learning}
(spherical and identical $\Sigma$, $\Delta_G \ge \Omega(d^{1/2})$, complexity $poly(d,k)$) and
followed by works of Dasgupta \& Schulman \cite{dasgupta2000two} (spherical and identical $\Sigma$,
$d\gg \log(k)$, $\Delta_G \ge \Omega(d^{1/4})$, complexity $poly(d,k)$), Arora \& Kannan
\cite{sanjeev2001learning} (general and identical $\Sigma$, $\Delta_G \ge \Omega(d^{1/4})$
complexity $O(k^d)$).

Instead of relying on the concentration of distance and use distance based clustering to learn the
GMM, we observe that in the well-separated case the characteristic function of the GMM has nice
properties, and one can exploit the concentration of the characteristic function to learn the
parameters.
Note that we do not impose any other assumption on the dimensions $k$ and $d$.

Next, we sketch the basic idea of applying the proposed super-resolution algorithm to learn
well-separated GMMs, guaranteeing that $N$ the required number of samples from the GMM, as well as
the computation complexity both are in the order of $poly(d,k)$.
Since $\sigma$ is a bounded scalar parameter, we can simply apply grid-search to find the best
match.
In the following we assume that the $\sigma$ is given and focus on learning the mean vectors and the
mixing weights.

Evaluate the characteristic function of a $d$ dimensional Gaussian mixture $X$, with identical and
spherical covariance matrix $\Sigma = \sigma^2 I_{d\times d}$, at $s\in\R^{d}$:
\begin{align*}
  \phi_X(s) = \mbb E[e^{i<x,s>}] = \sum_{j\in[k]}w_j e^{-{1\over 2} \sigma^2 \|s\|_2^2 + i
    <\mu^{(j)}, s>}.
\end{align*}
Also we let $\wh \phi_X(s)$ denote the empirical characteristic function evaluated at $s$ based on
$N$ i.i.d. samples $\{x_1,\dots x_N\}$ drawn from this GMM:
\begin{align*}
  \wh \phi_X(s) ={1\over N} \sum_{l\in[N]}e^{i<x_l,s>}.
\end{align*}
Note that $|e^{i<x_l,s>}| = 1$ for all samples, thus we can apply Bernstein concentration inequality
to the characteristic function and argue that $|\wh \phi_X(s) - \phi_X(s)|\le O({1\over \sqrt{N}})$
for all $s$.

In order to apply the proposed super-resolution algorithm, define
\begin{align*}
  & f(s) = e^{{1\over 2} \sigma^2 \pi^2 \|s\|_2^2} \phi_X(\pi s) = \sum_{j\in[k]} w_j e^{i\pi
    <\mu^{(j)}, s>}, \quad \tx{and} \quad \wt f(s) = e^{{1\over 2} \pi^2 \sigma^2 \|s\|_2^2}\wh
  \phi_X(s).
\end{align*}
In the context of learning GMM, taking measurements of $\wt f(s)$ corresponding to evaluating the
empirical characteristic function at different $s$, for $\|s\|_\infty\le R$, where $R$ is the cutoff
frequency.
Note that this implies $\|s\|_2^2\le dR^2$.
Therefore, we have that with high probability the noise level $\epsilon_z$ can be bounded by
\begin{align*}
  \epsilon_z = \max_{\|s\|_\infty\le R} |f(s)-\wt f(s)| = O\lt( {e^{\sigma^2  d R^2} \over \sqrt{N}}\rt).
\end{align*}
In order to achieve stable recovery of the mean vector $\mu^{(j)}$'s using the proposed algorithm,
on one hand, we need the cutoff frequency $R = \Omega({1/ {\sigma \Delta_G}})$; on the other hand,
we need the noise level $\epsilon_z = o(1)$.  It suffices to require $\sigma^2 d
R^2 = o(1)$, namely having large enough separation $\Delta_G\ge \Omega(d^{1/2})$.
In summary, when the separation condition is satisfied, to achieve target accuracy in estimating the
parameters, we need the noise level $\epsilon_z$ to be upper bounded by some inverse polynomial in
the dimensions, and this is equivalent to requiring the number of samples from the GMM to be lower
bounded by $poly(k,d)$.

Although this algorithm does not outperform the scaling result in
Dasgupta\cite{dasgupta1999learning}, it still sheds light on a different approach of learning GMMs.
We leave it as future work to apply super-resolution algorithms to learn more general cases of GMMs
or even learning mixtures of log-concave densities.




\subsection{Open problems}
In a recent work, Chen \& Chi \cite{chen2014robust} showed that via structured matrix completion,
the sample complexity for stable recovery can be reduced to $O(k\log^4 d)$. However, the computation
complexity is still in the order of $O(k^d)$ as the Hankel matrix is of dimension $O(k^d)$ and a
semidefinite program is used to complete the matrix.
It remains an open problem to reduce the sample complexity of our algorithm from $O(k^2)$ to the
information theoretical bound $O(k)$, while retaining the polynomial scaling of the computation
complexity.

Recently, Schiebinger et al \cite{schiebinger2015superresolution} studied the problem of learning a
mixture of shifted and re-scaled point spread functions $f(s) = \sum_{j} w_j\varphi(s, \mu^{(j)})$.
This model has the Gaussian mixture as a special case, with the point spread function being Gaussian
point spread $\varphi(s, \mu^{(j)}) = e^{-(s-\mu^{(j)})^\top \Sigma_j^{-1} (s-\mu^{(j)})}$.
We have discussed the connection between super-resolution and learning GMM.  Another interesting
open problem is to generalize the proposed algorithm to learn mixture of broader classes of
nonlinear functions.

\newpage
\subsection*{Acknowledgments}
The authors thank Rong Ge and Ankur Moitra for very helpful discussions.

\noindent
Sham Kakade acknowledges funding from the Washington Research  Foundation for innovation in
Data-intensive Discovery.

\bibliography{multi-prony}
\bibliographystyle{abbrv}


\section*{Auxiliary lemmas}

\begin{lemma}[Matrix Hoeffding]
  \label{lem:matrix-hoeffding}
  Consider a set $\{X^{(1)},\dots, X^{(m)}\}$ of independent, random, Hermitian matrices of dimension
  $k\times k$, with identical distribution $X$. Assume that $\mbb E[X]$ is finite, and $ X^2 \preceq
  \sigma^2 I$ for some positive constant $\sigma$ almost surely, then, for all $\epsilon \ge 0$,
  \begin{align*}
    Pr\lt( \lt\| {1\over m} \sum_{i=1}^m X^{(i)} - \mbb E[X]\rt \|_2 \ge \epsilon\rt) \le k
    e^{-{m^2\epsilon^2\over 8\sigma^2}}.
  \end{align*}

\end{lemma}

\begin{lemma}[Gershgorin's Disk Theorem]
  \label{lem:geshgorin-disk}
 The eigenvalues of a matrix $Y\in\C^{k\times k}$ are all contained in the following union of disks in the complex
 plane: $\cup_{j=1}^{k} \mc D(Y_{j,j}, R_j)$, where disk $\mc D(a,b) = \{x\in\C^{k}: \|x-a\|\le b\}$
 and $R_j= \sum_{j'\neq j}|Y_{j,j'}|$.
\end{lemma}

\begin{lemma}[Vector Random Projection]
  \label{lem:anti-concentration}
  Let $a\in\R^{m}$  be a random vector distributed uniformly over $\mc P^{m}_{1,2 }$, and
  fix a vector $v\in\C^{m}$. For $\delta\in(0,1)$, we have:
  \begin{align*}
    Pr\lt( |<a, v>| \le {\|v\|_2  \over \sqrt{em}}\delta\rt) \le \delta
  \end{align*}
\end{lemma}
\begin{proof}
  This follows the argument of Lemma 2.2 from Dasgupta \& Gupta \cite{dasgupta2003elementary}.  Extension to complex
  number is straightforward as we can bound the real part and the imaginary part separately.
\end{proof}



\end{document}